\documentclass[11pt]{article}
\usepackage[letterpaper]{geometry}
\usepackage[parfill]{parskip}
\usepackage{amsmath,amsthm,amssymb,bbm}
\usepackage{mathtools}
\usepackage{cases}

\usepackage{microtype}

\usepackage{subfigure}
\usepackage{algorithm,algorithmic}

\usepackage{color}
\usepackage{appendix}
\usepackage{xspace}
\usepackage{enumitem}

\usepackage[english]{babel}

\usepackage{authblk}

\usepackage{url}
\usepackage[authoryear]{natbib}
\usepackage[colorlinks,citecolor=blue,urlcolor=blue,linkcolor=blue,linktocpage=true]{hyperref}
\usepackage{cleveref}
\crefformat{equation}{(#2#1#3)}
\crefrangeformat{equation}{(#3#1#4) to~(#5#2#6)}
\crefname{equation}{}{}
\Crefname{equation}{}{}

\crefname{definition}{\textbf{definition}}{definitions}
\Crefname{definition}{Definition}{Definitions}
\crefname{assumption}{\textbf{assumption}}{assumptions}
\Crefname{assumption}{Assumption}{Assumptions}

\definecolor{maroon}{RGB}{192,80,77}

\newtheorem{theorem}{Theorem}
\newtheorem{lemma}[theorem]{Lemma}
\newtheorem{proposition}[theorem]{Proposition}

\newtheorem{definition}[theorem]{Definition}

\newtheorem{remark}[theorem]{Remark}

\newcommand{\OBP}{\textsc{ObjPert} }
\newcommand{\OPS}{\textsc{OPS} }

\newcommand{\vct}{\boldsymbol }

\newcommand{\R}{\mathbb{R}}

\def\E{\mathbb{E}}
\def\P{\mathbb{P}}
\def\Cov{\mathrm{Cov}}

\def\R{\mathbb{R}}
\def\I{\mathbb{I}}
\def\cA{\mathcal{A}}
\def\cB{\mathcal{B}}

\def\cN{\mathcal{N}}
\def\cP{\mathcal{P}}

\def\cX{\mathcal{X}}

\begin{document}

%\title{Privacy, Stability and Learnability }
\title{Privacy for Free: Posterior Sampling and Stochastic Gradient Monte Carlo}

\author[1]{Yu-Xiang Wang}
\author[2,1]{Stephen E. Fienberg}
\author[1,3]{Alex Smola}
\affil[1]{Machine Learning Department, Carnegie Mellon University}
\affil[2]{Department of Statistics, Carnegie Mellon University}
\affil[3]{Marianas Labs Inc.}

\maketitle

\begin{abstract}
We consider the problem of Bayesian learning on sensitive datasets and present two simple but somewhat surprising results that connect Bayesian learning to ``differential privacy'', a cryptographic approach to protect individual-level privacy while permiting database-level utility. Specifically, we show that that under standard assumptions, getting one single sample from a posterior distribution is differentially private ``for free''. We will see that estimator is statistically consistent, near optimal and computationally tractable whenever the Bayesian model of interest is consistent, optimal and tractable.
Similarly but separately, we show that a
recent line of works that use stochastic gradient for Hybrid Monte Carlo (HMC) sampling also preserve differentially privacy with minor or no modifications of the algorithmic procedure at all, these observations lead to an ``anytime'' algorithm for Bayesian learning under privacy constraint. We demonstrate that it performs much better than the state-of-the-art differential private methods on synthetic and real datasets.
\end{abstract}

\newpage
\tableofcontents
\newpage

\section{Introduction}

Bayesian models have proven to be one of the most successful classes of tools in machine learning. It stands out as a principled yet conceptually simple pipeline for combining expert knowledge and statistical evidence, modelling with complicated dependency structures and harnessing uncertainty by making probabilistic inferences \citep{geman1984stochastic,gelman2014bayesian}. In the past few decades, the Bayesian approach has been intensively used in modelling speeches \citep{rabiner1989tutorial}, text documents \citep{blei2003latent}, images/videos \citep{fei2005bayesian}, social networks \citep{airoldi2009mixed}, brain activity \citep{penny2011statistical}, and is often considered gold standard in many of these application domains. Learning a Bayesisan model typically involves sampling from a posterior distribution, therefore the learning process is inherently randomized.

Differential privacy (DP) is a cryptography-inspired notion of privacy \citep{dwork2006differential,dwork2006calibrating}. It is designed to provide a very strong form of protection of individual user's private information and at the same time allow data analyses to be conducted with proper utility. Any algorithm that preserves differential privacy must be appropriately randomized too. For instance, one can differential-privately release the average salary of Californian males by adding a Laplace noise proportional to the sensitivity of this figure upon small perturbation of the data sample.

In this paper, we connect the two seemingly unrelated concepts by showing that under standard assumptions, the intrinsic randomization in the Bayesian learning can be exploited to obtain a degree of differential privacy. In particular, we show that:
%Our key contributions are summarized below:
\begin{itemize}[itemsep=0mm]
\item   Any algorithm that produces a single sample from the exact (or approximate) posterior distribution of a Bayesian model with bounded log-likelihood is $\epsilon$ (or $(\epsilon,\delta)$)-differentially private\footnote{Similar observations were made in \citet{mir13} and \citet{dimitrakakis2014robust} under slightly different regimes and assumptions, and we will review them among other related work in Section~\ref{sec:relatedwork}.}. By the classic results in asymptotic statistics \citep{le1986bernstein,vandervaart2000asymptotic}, we show that this posterior sample is a consistent estimator whenever the Bayesian model is consistent; and near optimal whenever the asymptotic normality and efficiency of the maximum likelihood estimate holds.
%In addition, if we calibrate the posterior for specific privacy loss $\epsilon$, it matches the information-theoretic lowerbound for parametric estimation under differential privacy.

\item The popular large-scale sampler Stochastic Gradient Langevin Dynamics~\citep{welling2011bayesian} and extensions, e.g. \citet{ahn2012bayesian,Chen2014SGHMC,ding2014bayesian} obey $(\epsilon,\delta)$-differentially private with no algorithmic changes when the stepsize is chosen to be small. This gives us a procedure that can potentially output many (correlated) samples from an approximate posterior distribution.
%In fact, the smaller these stepsizes are the more accurate the stationary distributions are for posterior distribution.
\end{itemize}

These simple yet interesting findings make it possible for differential privacy to be explicitly considered when designing Bayesian models, and for Bayesian posterior sampling to be used as a valid DP mechanism. We demonstrate empirically that these methods work as well as or better than the state-of-the-art differential private empirical risk minimization (ERM) solvers using objective perturbation~\citep{chaudhuri2011differentially, kifer2012private}.

The results presented in this paper are closely related to a number of previous work, e.g., \citet{mcsherry2007mechanism,mir13,bassily2014private,dimitrakakis2014robust}. Proper comparisons with them would require the knowledge of our results, thus we will defer detailed comparisons to Section~\ref{sec:relatedwork} near the end of the paper.

%and DP-Stochastic Gradient Descent~\citep{bassily2014private}.

\section{Notations and Preliminary}

Throughout the paper, we assume data point $\vct x\in \cX$ and $\vct \theta\in\Theta$ is the model. This can be the finite dimensional parameter of a single exponential family model or a collection of these in a graphical model, or a function in a Hilbert space or other infinite dimensional objects if the model is nonparametric. $\pi(\vct \theta)$ denotes a prior blief of the model parameters and $p(\vct x|\vct \theta)$ and $\ell(\vct x|\vct \theta)$ are the likelihood and log-likelihood of observing data point $x$ given model parameter $\vct \theta$. If we observe $X=\{\vct x_1,...,\vct x_n\}$, the posterior distribution
$$
\pi(\vct \theta|X) = \frac{\pi(\vct \theta)\prod_{i=1}^N p(\vct x_i|\vct \theta)}{\int \prod_{i=1}^N p(\vct x_i|\vct \theta) \pi(\vct \theta) d\pi}
$$
denotes the updated belief conditioned on the observed data. Learning Bayesian models correspond to finding the mean or mode of the posterior distribution, but often, the entire distribution is treated as the output, which provides much richer information than just a point estimator. In particular, we get error bars of the estimators for free (credibility intervals).

%Statistically, Bayesian learning is controversal because it does not always have frequentist consistency, efficiency and coverage, especially  for nonparametric models. We will discuss about this further later in Section~\ref{sec:cons_nearoptimality} when we discuss the consistency of our one-posterior sample (OPS) estimator.

Ignoring the philosophical disputes of Bayesian methods for the moment, practical challenges of Bayesian learning are often computational. As the models get more complicated, often there is not a closed-form expression for the posterior. Instead, we often rely on Markov Chain Monte Carlo methods, e.g., Metropolis-Hastings algorithm~\citep{hastings1970monte} to generate samples. This is often prohibitively expensive when the data is large. One recent approach to scale up Bayesian learning is to combine stochastic gradient estimation as in \citet{robbins1951stochastic} and Monte Carlo methods that simulates stochastic differential equations, e.g. \citet{neal2011mcmc}. These include Stochastic Gradient Langevin dynamics (SGLD) \citep{welling2011bayesian}, Stochstic Gradient Fisher scoring (SGFS) \citep{ahn2012bayesian}, Stochastic Gradient Hamiltonian Monte Carlo (SGHMC) \citep{Chen2014SGHMC} as well as more recent Stochastic Gradient Nos\'{e}-Hoover Thermostat (SGNHT) \citep{ding2014bayesian}. We will describe them with more details and show that these series of tools provide differential privacy as a byproduct of using stochastic gradient and requiring the solution to not collapse to a point estimate.

\subsection{Differential privacy}
Now we will talk about what we need to know about differential privacy. Let the space of data be $\cX$ and data points $X,Y\in \cX^n$.
Define $d(X,Y)$ to be the edit distance or Hamming distance between data set $X$ and $Y$, for instance, if $X$ and $Y$ are the same except one data point, then $d(X,Y)=1$.
\begin{definition}(Differential Privacy)
We call a randomized algorithm $\cA$ $(\epsilon,\delta)$-differentially private with domain $\cX^n$ if for all measurable set $S\subset \text{Range}(\cA)$ and for all $X,Y \in \cX^{n}$ such that $d(X,Y)\leq 1$, we have
$$
\P(\cA(X)\in S) \leq \exp(\epsilon) \P(\cA(Y)\in S) + \delta.
$$
If $\delta=0$, then $\cA$ is the called $\epsilon$-differential private.
\end{definition}
%where the probability is over the randomness of $\cA$.
This definition naturally prevents linkage attacks and the identification of individual data from adversaries having arbitrary side information and infinite computational power. The promise of differential privacy has been interpreted in statistical testing, Byesian inference and information theory for which we refer readers to Chapter 1 of \citep{dwork2013algorithmic}.

There are several interesting properties of differential privacy that we will exploit here. Firstly, the definition is closed under postprocessing.
\begin{lemma}[Postprocessing immunity]\label{lemma:postprocessing}
If $\cA$ is an $(\epsilon,\delta)$-DP algorithm, $\cB\circ \cA$ is also $(\epsilon,\delta)$-DP algorithm for any $\cB$.
\end{lemma}
This is natural because otherwise the whole point of differential privacy will be forfeited. Also, the definition automatically allows for cases when the sensitive data are accessed more than once.
\begin{lemma}[Composition rule]\label{lemma:composition}
If algorithm $\cA_1$ is $(\epsilon_1,\delta_1)$-DP, and $\cA_2$ is $(\epsilon_2,\delta_2)$-DP then $(\cA_1\otimes\cA_2)$ is $(\epsilon_1+\epsilon_2,\delta_1+\delta_2)$-DP.
\end{lemma}
We will describe more advanced properties of DP as we need in Section~\ref{sec:sgld}.

%A more general DP mechanism is the exponential mechanism, which output a hypothesis with probability proportional to an exponentiated utility function. This casts many private mechanisms such as Laplace noise perturbation into special cases.
%\begin{theorem}[Exponential Mechanism \citep{mcsherry2007mechanism}]
%Let $q(X,\vct \theta)$ be a utility function satisfying $\sup_{d(X,X')\leq 1} |q(X,\vct \theta) - q(X',\vct \theta)| \leq \Delta q$. Then the algorithm that output $\vct \theta\in\Theta$ with $\P(\vct \theta) \propto \exp(\frac{\epsilon}{2\Delta q}q(X,\vct \theta))$ preserves $\epsilon$-DP.
%\end{theorem}
%The one posterior sample mechanism that we will propose in the next section is essentially also a special case of the exponential mechanism.

\section{Posterior sampling and differential privacy}
In this section, we make a simple observation that under boundedness condition of a log-likelihood, getting one single sample from the posterior distribution (denoted by ``\OPS mechanism'' from here onwards) preserves a degree of differential privacy for free. Then we will cite classic results in statistics and show that this sample is a consistent estimator in a Frequentist sense and near-optimal in many cases.

%From here onwards, we will refer to this estimator as One-Posterior Sample estimator and the corresponding
\subsection{Implicitly Preserving Differential Privacy}

To begin with, we show that sampling from the posterior distribution is intrinsically differentially private.
\begin{theorem}\label{thm:private_posterior}
If $\sup_{\vct x\in \cX,\vct \theta\in\Theta}\left|\log p(\vct x|\vct \theta)\right| \leq B$, releasing one sample from the posterior distribution $p(\vct \theta|X^n)$ with any prior preserves $4B$-differential privacy. Alternatively, if $\cX$ is a bounded domain (e.g., $\|x\|_*\leq R \;\forall \vct x\in \cX$) and $\log p(\vct x|\vct \theta)$ is an $L$-Lipschitz function in $\|\cdot\|_*$ for any $\vct \theta\in \Theta$, then releasing one sample from the posterior distribution preserves $4LR$-differential privacy.
\end{theorem}
\begin{proof}
The posterior distribution $p(\vct \theta|\vct x_1,...,\vct x_n) = \frac{\prod_{i=1}^n p(\vct x_i|\vct \theta)p(\vct \theta)}{\int_{\vct \theta} \prod_{i=1}^n p(\vct x_i|\vct \theta)p(\vct \theta) d\vct \theta}$. For any $\vct x_1,...,\vct x_n, x'_k$,
The ratio can be factorized into
\begin{align*}
\frac{p(\vct \theta|\vct x_1,...,x'_k,...,\vct x_n)}{p(\vct \theta|\vct x_1,...,\vct x_k,...,\vct x_n)} = \underbrace{\frac{p(\vct x_k'|\vct \theta)\prod_{i=1:n, i\neq k} p(\vct x_i|\vct \theta)p(\vct \theta)}{\prod_{i=1}^n p(\vct x_i|\vct \theta)p(\vct \theta)}}_\text{Factor 1}\times \underbrace{\frac{\int_{\vct \theta} \prod_{i=1}^n p(\vct x_i|\vct \theta)p(\vct \theta) d\vct \theta}{ \int_{\vct \theta} p(\vct x_k'|\vct \theta)\prod_{i=1:n, i\neq k} p(\vct x_i|\vct \theta)p(\vct \theta)d\vct \theta }}_\text{Factor 2}.
\end{align*}
\vspace{-2em}

It follows that
\begin{align*}
\text{Factor 1} &= \frac{p(\vct x'_k|\vct \theta)}{p(\vct x_k|\vct \theta)} = e^{\log p(\vct x'_k|\vct \theta) -\log p(\vct x_k|\vct \theta)} \leq e^{2B},\\
\text{Factor 2} &= \frac{\int_{\vct \theta} \prod_{i\neq k} p(\vct x_i|\vct \theta)p(\vct \theta) p(\vct x_k)d\vct \theta}{ \int_{\vct \theta} p(\vct x_k'|\vct \theta)\prod_{i\neq k} p(\vct x_i|\vct \theta)p(\vct \theta)d\vct \theta }
= \frac{\int_{\vct \theta} \prod_{i\neq k} p(\vct x_i|\vct \theta)p(\vct \theta) p(\vct x_k'|\vct \theta) \frac{p(\vct x_k)}{p(\vct x_k')} d\vct \theta}{ \int_{\vct \theta} p(\vct x_k'|\vct \theta)\prod_{i\neq k} p(\vct x_i|\vct \theta)p(\vct \theta)d\vct \theta }\\
&= \frac{\int_{\vct \theta} \prod_{i\neq k} p(\vct x_i|\vct \theta)p(\vct \theta) p(\vct x_k'|\vct \theta) e^{\log p(\vct x_k|\vct \theta) - \log p(\vct x_k'|\vct \theta)} d\vct \theta}{ \int_{\vct \theta} p(\vct x_k'|\vct \theta)\prod_{i\neq k} p(\vct x_i|\vct \theta)p(\vct \theta)d\vct \theta }\\
&\leq e^{2B} \frac{m(\vct x_1,...,\vct x_k',...,\vct x_n)}{m(\vct x_1,...,\vct x_k',...,\vct x_n)} = e^{2B}.
\end{align*}
where we use $m(X)$ to denote the marginal distribution. As a result, the whole thing is bounded by $e^{4B}$.

Alternatively, we can use the Lipschitz constant and boundedness to get $\log p(\vct x'_k|\vct \theta) -\log p(\vct x_k|\vct \theta) \leq L\|x'_k-\vct x_k\|_* \leq 2LR$.
\end{proof}

%\paragraph{Exponential mechanism for approximate maximum likelihood.}
Readers familiar with differential privacy must have noticed that this is actually an instance of the exponential mechanism~\citep{mcsherry2007mechanism}, a general procedure that preserves privacy while making outputs with higher utility exponentially more likely. If one sets the utility function to be the log-likelihood and the privacy parameter being $4B$, then we get exactly the one-posterior sample mechanism.  This exponential mechanism point of view provides an an simple extension which allows us to specify $\epsilon$ by simply scaling the log-likelihood (see Algorithm~\ref{alg:OPS}). We will overload the notation \OPS to also represent this mechanism where we can specify $\epsilon$. The nice thing about this algorithm is that there is almost zero implementation effort to extend all posterior sampling-based Bayesian learning models to have differentially privacy of any specified $\epsilon$.

\begin{algorithm} [tb]                   % enter the algorithm environment
\caption{One-Posterior Sample (\OPS) estimator}          % give the algorithm a caption
\label{alg:OPS}                           % and a label for \ref{} commands later in the document
\begin{algorithmic}                    % enter the algorithmic environment
    \INPUT{ Data $X$, $\log$-likelihood function $\ell(\cdot|\cdot)$ satisfying $\sup_{\vct x, \vct \theta}\|\ell(\vct x|\vct \theta)\| \leq B$  a prior $\pi(\cdot)$. Privacy loss $\epsilon$.}
    \STATE{1.} Set $\rho= \min\{1,\frac{\epsilon}{4B}\}$.
    \STATE{2.} Re-define $\log$-likelihood function and the prior
    $\ell'(\cdot|\cdot) := \rho\ell(\cdot|\cdot)$ and $\pi'(\cdot) :=(\pi(\cdot) )^{\rho}.$
 \OUTPUT{ $\hat{ \vct \theta} \sim P(\vct \theta| X) \propto \exp\left(\sum_{i=1}^N \ell'(\vct \theta|\vct x_i)\right)\pi'(\vct \theta)$.}
\end{algorithmic}
\end{algorithm}

\paragraph{Assumption on the boundedness.}
The boundedness on the loss-function (log-likelihood here) is a standard assumption in many DP works \citep{chaudhuri2011differentially,bassily2014private,song2013stochastic,kifer2012private}.
Lipschitz constant $L$ is usually small for continuous distributions (at least when the parameter space $\Theta$ is bounded). This is a bound on $\log p(\vct x|\vct \theta)$) so as long as $p(\vct x|\vct \theta)$ does not increase or decrease super exponentially fast at any point, $L$ will be a small constant. $R$ can also be made small by a simple preprocessing step that scales down all data points. In the aforementioned papers that assume $L$, it is typical that they also assume $R=1$ for convenience. So we will do the same. In practice, we can
algorithmically remove large data points from the data by some predefined threshold or using the ``Propose-Test-Release'' framework in \citep{dwork2009differential} or perform weighted training where we can assign lower weight to data points with large magnitude. Note that this is a desirable step for the robustness to outliers too. Exponential families (in Hilbert space) are an example, see e.g. \citet{bialek2001predictability,hofmann2008kernel,wainwright2008graphical}.

%Just to give a few examples:
%\begin{example}
%\begin{itemize}
%\item Logistic loss (Bernoulli model) $\log(1+\exp(\vct \theta^T (\vct x y)))$. It is 1-Lipschitz in $\vct \theta^T\vct x$ and therefore, $\|\vct \theta\|$-Lipschitz in $(\vct x,y)$.
%\item Soft-Max (Categorical model) $\log\left(\frac{e^{\vct \theta_j^T\vct x}}{\sum_{k=1}^K e^{\vct \theta_k^T\vct x}}\right)$
%\item Square loss (Gaussian model) $(y-\vct \theta^T\vct x)^2$. This is  $\|\vct \theta\|$-Lipschitz.
%\item Hinge loss (SVM) $\max(0,-\vct \theta^T\vct x y)$. It is 1-Lipschitz in $\vct \theta^T\vct x y$ and therefore, $\|\vct \theta\|$-Lipschitz in $(\vct x,y)$.
%\item Absolute value loss (Robust regression) $|y - \theta^T\vct x|$. It is 1-Lipschitz in $\vct \theta^T\vct x y$ and therefore, $\|\vct \theta\|$-Lipschitz in $(\vct x,y)$.
%\end{itemize}
%\end{example}

%or compute ir in a data-dependent fashion using the ``Propose-Test-Release'' framework in Jing and Dwork (2009).

%Lastly, we can also adaptively screens out the bulk of $\Theta$ as in the ``localization'' step in Raef's paper. This would help us to require a condition much weaker than ``uniformly over all $\vct \theta\in\Theta$''. We can use a pre-defined region of $\vct \theta$, or we can come up with a region in a data dependent fashion and then use the ``Propose-Test-Release'' framework somehow.

\subsection{Consistency and Near-Optimality}\label{sec:cons_nearoptimality}
Now we move on to study the consistency of the \OPS estimator. In great generality, we will show that the one-posterior sample estimator is consistent whenever the Bayesian model is posterior consistent. Since the consistency in Bayesian methods can have different meanings, we briefly describe two of them according to the nomenclature in \citet{orbanz2012lecture}.
\begin{definition}[Posterior consistency in the Bayesian Sense]\label{def:pos_cons1}
 For a prior $\pi$, we say the model is posterior consistent in the Bayesian sense,  if $\vct \theta \sim \pi(\vct \theta)$, $\vct x_1,...,\vct x_n \sim p_{\vct \theta}$, and the posterior
 \vspace{-1em}
 $$\pi(\vct \theta|\vct x_1,...,\vct x_n) \overset{\text{weakly}}{\longrightarrow} \delta_{\vct \theta}  \text{ a.s. } \pi.$$
 $\delta_{\vct \theta}$ is the dirac-delta function at $\vct \theta$.
\end{definition}
In great generality, Doob's well-known theorem guarantees posterior consistency in the Bayesian sense for a model with any prior under no conditions except identifiability and measurability. A concise statement of Doob's result can be found in \citet[Theorem 10.10]{vandervaart2000asymptotic}).
  %under almost no conditions besides standard measurability and identifiability conditions.

An arguably more reasonable definition is given below. It applies to the case when the statistician who chooses the prior $\pi$ does not know about the true parameter.
\begin{definition}[Posterior consistency in the Frequentist Sense]\label{def:pos_cons2}
 For a prior $\pi$, we say the model is posterior consistent in the Frequentist sense,  if for every $\vct \theta_0\in\Theta$, $\vct x_1,...,\vct x_n \sim p_{\vct \theta}$, the posterior
 $$\pi(\vct \theta|\vct x_1,...,\vct x_n) \overset{\text{weakly}}{\longrightarrow} \delta_{\vct \theta_0}  \text{ a.s. }p_{\vct \theta_0}. $$
\end{definition}
This type of consistency is much harder to satisfy especially when $\Theta$ is an infinite dimensional space, in which case the consistency often depends on the specific priors to use. A promising series of results on the consistency for Bayesian nonparametric models can be found in  \citet{ghosal2010dirichlet}).

Regardless which definition one favors, the key notion of consistency is that the posterior distribution to concentrates around the true underlying $\vct \theta$ that generates the data.
\begin{proposition}
The one-posterior sample estimator is consistent \emph{if and only if} the Bayesian model is posterior consistent (in either Definition~\ref{def:pos_cons1}~or~\ref{def:pos_cons2} ).
\end{proposition}
\begin{proof}
The equivalence follows from the standard equivalence of convergence weakly and convergence in probability when a random variable converges weakly to a point mass.
\end{proof}

How about the rate of convergence? In the low dimensional setting when $\vct \theta\in\Theta\subset \R^d$ and $p_{\vct \theta}(\vct x)$ is suitably differentiable and the prior is supported at the neighborhood of the true parameter, then by the Bernstein-von Mises theorem \citep{le1986bernstein}, the posterior mean is an asymptotically efficient estimator and the posterior distribution converges in $L_1$-distance to a normal distribution with covariance being the inverse Fisher Information.
\begin{proposition}\label{prop:efficiency}
Under the regularity conditions where Bernstein-von Mises theorem holds, the One-Posterior sample $\hat{\vct \theta} \sim \pi(\vct \theta|\vct x_1,..,\vct x_n) $ obeys
$$
\sqrt{n}(\hat{\vct \theta} - \vct \theta_0) \overset{\text{weakly}}{\longrightarrow}  \cN(0, 2\I^{-1}),
$$
i.e., the One-Posterior sample estimator has an asymptotic relative efficiency of 2.
\end{proposition}
\begin{proof}
Let the One-Posterior sample $\hat{\vct \theta} \sim \pi(\vct \theta|\vct x_1,..,\vct x_n) $. By Bernstein-von Mises theorem
$
\sqrt{n}(\hat{\vct \theta} - \tilde{\vct \theta}) \overset{\text{weakly}}{\rightarrow}  \cN(0, \I^{-1}).
$
%$$
%\| \pi(\vct \theta|\vct x_1,..,\vct x_n) - \cN(\tilde{\vct \theta},I_n^{-1})\|_{L_1}\rightarrow 0
%$$
By the asymptotical normality and efficiency of the posterior mean estimator  $\sqrt{n}(\tilde{\vct \theta} -\vct \theta_0) \overset{\text{weakly}}{\rightarrow}  \cN(0, \I^{-1}).$
%The asymptotic marginal distribution of the one-posterior sample
%$$p(\hat{\vct \theta}) = \int p(\hat{\vct \theta} | \tilde{\vct \theta}) d p(\tilde{\vct \theta})=p(\hat{\vct \theta}-\tilde{\vct \theta}) p(\tilde{\vct \theta}).$$
The proof is complete by taking the sum of the two asymptotically independent Gaussian vectors ($\tilde{\vct \theta}$ and $\hat{\vct \theta}-\tilde{\vct \theta}$ are asymptotically independent).
\end{proof}
The above proposition suggests that in many interesting classes of parametric Bayesian models, the One-Posterior Sample estimator is asymptotically near optimal. Similar statements can also be obtained for some classes of semi-parametric and nonparametric Bayesian models \citep{ghosal2010dirichlet}, which we leave as future work.%but let us not deviate from the focus of this paper.

The drawback of the above two propositions is that they are only stated for the version of the OPS when $\epsilon=4B$. Using results in \citet{de2013bayesian} and \citet{kleijn2012bernstein} for Bayesian learning under misspecified models, we can prove consistency, asymptotic normality for any $\epsilon$ and parameterize the asymptotic relative efficiency of the \OPS estimator as a function of $\epsilon$. The key idea is that when scaling the log-likelihood and sample from a different distribution, we are essentially fitting a model that may not include the data-generating true distribution.  \citet{de2013bayesian} shows that under mild conditions, when the model is misspecified, the posterior distribution will converge to a point mass $\vct \theta^*$ that minimizes the KL-divergence between between the true distribution and the corresponding distribution in the misspecified model. $\vct \theta^*$ is essentially MLE and in our case, since we only scaled the distribution, the MLE will remain exactly the same. \citet{de2013bayesian}'s result is quite general and covers both parametric and nonparametric Bayesian models and whenever their assumptions hold, the \OPS estimator is consistent. Using a similar argument and the modifed Bernstein-Von-Mises theorem in \citet{kleijn2012bernstein}, we can prove asymptotic normality and near optimality for the subset of problems where regularities of MLE hold.
\begin{proposition}\label{prop:general_eps}
Under the same assumption as Proposition~\ref{prop:efficiency}, if we set a different $\epsilon$ by rescaling the log-likelihood by a factor of $\frac{\epsilon}{4B}$, then the the One-Posterior sample estimator obeys $$
\sqrt{n}(\hat{\vct \theta} - \vct \theta_0) \overset{\text{weakly}}{\longrightarrow}  \cN\left(0, (1+\frac{4B}{\epsilon})\I^{-1}\right),
$$
in other word, the estimator has an ARE of $(1+\frac{4B}{\epsilon})$.
\end{proposition}
\begin{proof}
By scaling the log-likelihood, we are essentially changing the correct model $p_{\vct \theta}$ to a misspecified model $(p_{\vct \theta})^{\frac{\epsilon}{4B}}$. Let the true log-likelihood be $\ell$ and the misspecified log-likelihood be $\tilde{\ell} = \frac{\epsilon}{4B} \ell$, in addition, define
$$V(\vct \theta) := \E_{\vct \theta} \nabla\tilde{\ell}(\vct \theta)\nabla\tilde{\ell}(\vct \theta)^T = \frac{\epsilon^2}{16 B^2}\E_{\vct \theta} \nabla\ell(\vct \theta) \nabla\ell(\vct \theta)^T = \frac{\epsilon^2}{16 B^2}\I(\vct \theta)$$
$$J(\vct \theta) :=  -\E_{\vct \theta} \nabla^2\tilde{\ell}(\vct \theta)= -\frac{\epsilon}{4 B}\E_{\vct \theta} \nabla^2\ell(\vct \theta) = -\frac{\epsilon}{4 B} \I(\vct \theta).$$
The last equality holds under the standard regularity conditions. By the sandwich formula, the maximum likelihood estimator $\hat{\vct \theta}$ under the misspecified model is asymptotically normal:
$$\sqrt{n}(\hat{\vct \theta} - \vct \theta^*) \overset{\text{weakly}}{\rightarrow} \cN(0, J^{-1} V J(-1)) =\cN(0, \I^{-1}) $$
where $\vct \theta^*$ defines the closest (in terms of KL-divergence) model in the misspecified class of distributions to the true distribution that generates the data. Since the difference is only in scaling, the minimum KL-divergence is obtained at $\vct \theta^* = \vct \theta$. Now under the same regularity conditions, we can invoke the modified Bernstein-Von-Mises theorem for misspecified models \citep[Lemma~2.2]{kleijn2012bernstein}, which says that the posterior distribution $p(\theta|X^n)$ (of the misspecified model) converges in distribution to $\cN(\hat{\vct \theta}, (nJ)^{-1})$. In our case, $(nJ)^{-1} = \frac{4B}{n\epsilon}\I^{-1}$. The proof is concluded by noting that the posterior sample is an independent draw.
\end{proof}
We make a few interesting remarks about the result.
\begin{enumerate}
\item Proposition~\ref{prop:general_eps} suggests that for models with bounded log-likelihood, OPS is only a factor of $(1+4B/\epsilon)$ away from being optimal. This is in sharp contrast to most previous statistical analysis of DP methods that are only tight up to a numerical constant (and often a logarithmic term). In $\ell_2$-norm, the convergence rate is $O(\frac{\sqrt{1+4B/\epsilon}\|I^{-1}\|_F}{\sqrt{n}})$. The bound depends on the dimension through the Frobenius norm which is usually $O(\sqrt{d})$. The bound can be further sharpened using assumptions on the intrinsic rank, incoherence conditions or the rate of decays in eigenvalues of the Fisher information. In $\ell_\infty$-norm, the convergence rate is $\frac{\sqrt{1+4B/\epsilon}\|I^{-1}\|_2}{\sqrt{n}}$, which does not depend on the dimension of the problem.
\item Another implication is on statistical inference. Proposition~\ref{prop:general_eps} essentially generalizes that classic results in hypothesis testing and confidence intervals, e.g., Wald test, generalized likelihood ratio test, can be directly adopted for the private learning problems, with an appropriate calibration using $\epsilon$. We can control the type I error in an asymptotically exact fashion. In addition, the trade-off with $\epsilon$ and the test power is also explicitly described, so in cases where the power of the tests are well-studied \citep{lehmann2006testing}, the same handle can be used to analyze the most-powerful-test under privacy constraints.
\item Lastly, the results in \citet{de2013bayesian} and \citet{kleijn2012bernstein} are much more general. It is easy to extend the guarantee for \OPS to handle private Bayesian learning in a fully agnostic setting and in non-iid cases. We will leave the formalization of these claims as future directions.
\end{enumerate}

%\begin{proposition}
%Under the same assumption as Proposition~\ref{prop:efficiency}, if we set a different $\epsilon$ by rescaling the log-likelihood by a factor of $\rho$, then the the One-Posterior sample obeys
%$$
%\sqrt{n}(\hat{\vct \theta} - \vct \theta_0) \overset{\text{weakly}}{\longrightarrow}  \cN(0, 16B^2/\epsilon^2 \I^{-1} + \I^{-1}),
%$$
%\end{proposition}
%\begin{proof}
%The proof is only slightly trickier. The data are still drawn from $p_\theta(x)$ but we are sampling from a distribution that are flatter if $\epsilon<4B$ or higher if $\epsilon>4B$, as a result the maximum likelihood estimator, or the posterior mean will still be an asymptotically efficient estimator with the same distribution $\cN(0,\I^{-1})$. The asymptotic posterior distribution is going to be different. Its asymptotical normal distribution will have covariance matrix being the Fisher information associated with the scaled log-likelihood. To calculate it out, we find the Hessian of the log-likelihood and use linearity.
%\end{proof}

\subsection{(Efficient) sampling from approximate posterior}
The privacy guarantee in Theorem~\ref{thm:private_posterior} requires sampling from the exact posterior. In practice, however, exact samplers are rare. As Bayesian models get more and more complicated, often the only viable option is to use Markov Chain Monte Carlo (MCMC) samplers which are almost never exact. There are exceptions, e.g., \citet{propp1998coupling} but they only apply to problems with very special structures. A natural question to ask is whether we can still say something meaningful about privacy when the posterior sampling is approximate. It turns out that we can, and the level of approximation in privacy is the same as the level of approximation in the sampling distribution.
\begin{proposition}\label{prop:approx_sample}
If $\cA$ that sampling from distribution $P_{X}$ preserves $\epsilon$-differential privacy, then any approximate sampling procedures $\cA'$ that produces a sample from $P'_X$ such that $\|P_X-P'_X\|_{L_1}\leq \delta$ for any $X$ preserves $(\epsilon,(1+e^\epsilon)\delta)$-differential privacy.
\end{proposition}
\begin{proof}
For any $S\in \text{Range}(\cA')$, and $d(X,X')\leq 1$
\begin{align*}
\P\left(\cA'(X)\in S\right) &= \int_S dP'_X\leq \int_S dP_X +\delta\\
&\leq e^\epsilon \int_S  dP_{X'} +\delta \leq e^\epsilon \int_S  dP_{X'} \\
&\leq e^\epsilon \int_S  dP'_{X'} + (1+e^\epsilon)\delta\\
&=e^\epsilon \P\left(\cA'(X')\in S\right) + (1+e^\epsilon)\delta,
\end{align*}
This is $(\epsilon,(1+e^\epsilon)\delta)$-DP by definition.
\end{proof}
We are using $L_1$ distance of the distribution because it is a commonly accepted metric to measure the convergence rate MCMC \cite{rosenthal1995minorization}, and Proposition~\ref{prop:approx_sample} leaves a clean interface for computational analysis in determining the number of iterations needed to attain a specific level of privacy protection.

\paragraph{A note on computational efficiency.}
The (unsurprising) bad news is that even approximate sampling from the posterior is NP-Hard in general, see, e.g. \citet[Theorem~8]{sontag2011complexity}. There are however interesting results on when we can (approximately) sample efficiently. Approximation is easy for sampling LDA when $\alpha>1$ while NP-Hard when $\alpha<1$.  A more general result in \citet{applegate1991sampling} suggests that we can get a sample with arbitrarily close approximation in polynomial time for a class of near log-concave distributions. The log-concavity of the distributions would imply convexity in the log-likelihood, thus, this essentially confirms the computational efficiency of all convex empirical risk minimization problems under differential privacy constraint (see \citet{bassily2014private}).

The nice thing is that since we do not modify the form of the sampling algorithm at all, the \OPS algorithm is going to be a computationally tractable DP method whenever the Bayesian learning model of interest is proven to be computationally tractable.

This observation provides an interesting insight into the problem of computational lower bound of differential private machine learning. Unlike what is conjectured in \citet{dwork2014preserving}, our observation seems to suggest that the computational barrier is not specific to differential privacy, but rather the barrier of learning in general. The argument seems to hold at least for some class of problems, where the posterior sample achieves the optimal statistical rate and is at least $4B$-DP.

%Ignoring the rate, it will be a very interesting future direction to explore whether there exists a problem that is learnable in polynomial time but not efficiently learnable under differential privacy. The example in given by Hsu and Chaudhuri suggests that establish the gap between learnable and privately learnable problem is not a polynomial time algorithm.

\subsection{Discussions and comparisons}
 \OPS  has a number of advantages over the state-of-the-art differentially private ERM method: objective perturbation \citep{chaudhuri2011differentially,kifer2012private} (\OBP from here onwards). \OPS works with arbitrary bounded loss functions and priors while \OBP needs a number of restrictive assumptions including twice differentiable loss functions, strongly convexity parameter to be greater than a threshold and so on. These restrictions rule out many commonly used loss functions, e.g., $\ell_1$-loss, hinge loss, Huber function just to name a few.

 Also, \OBP's privacy guarantee holds only for the exact optimal solution, which is often hard to get in practice. In contrast, \OPS works when the sample is drawn from an approximate posterior distribution. From a practical point of view, since \OPS stems from the intrinsic privacy protection of Bayesian learning, it requires very little implementation effort to deploy it for practical applications. It also requires the problem to be strong convexity with a minimum strong convexity parameter. When the condition is not satisfied, \OBP will need to add additional quadratic regularization to make it so, which may bias the problem unnecessarily.

%and to get an arbitrarily close approximation in polynomial time for

%Applegate and Kannan.

\section{Stochastic Gradient MCMC and $(\epsilon,\delta)$-Differential privacy}\label{sec:sgld}
Given a fixed privacy budget, we see that the single posterior sample produces an optimal point estimate, but what if we want multiple samples? Can we use the privacy budget in a different way that produces many approximate posterior samples? %Also, what if there is no guarantee that the posterior sampler work at all?

 In this section we will provide an answer to it by looking at a class of Stochastic Gradient MCMC techniques developped over the past few years. We will show that they are also differentially private for free if the parameters are chosen appropriately.

The idea is to simply privately release an estimate of the gradient (as in \citet{song2013stochastic,bassily2014private}) and leverage upon the following two celebrated lemmas in differential privacy in the same way as \citet{bassily2014private} does in deriving the near-optimal $(\epsilon,\delta)$-differentially private SGD.

%A more sophisticated result shows that we can trade off a small amount of $\delta$ to get a much better bound for the privacy loss due to composition.

The first lemma is the advanced composition which allows us to trade off a small amount of $\delta$ to get a much better bound for the privacy loss due to composition.
\begin{lemma}[Advanced composition, c.f.,Theorem 3.20 in \citep{dwork2013algorithmic}]\label{lemma:adv_composition}
For all $\epsilon,\delta,\delta'\geq 0$, the class of $(\epsilon,\delta)$-DP mechanisms satisfy $(\epsilon',k\delta+\delta')$-DP under $k$-fold adaptive composition for:
$$
\epsilon'=\sqrt{2k\log(1/\delta')} \epsilon + k \epsilon(e^{\epsilon}-1).
$$
\end{lemma}
\begin{remark}\label{rmk:adv_composition}
When $\epsilon = \frac{c}{\sqrt{2k\log(1/\delta')}}<1$ for some constant $c<\sqrt{\log(1/\delta')}$, we can simplify the above expression into
$
\epsilon' \leq  2c.
$
%c\big(1+2/(2\log(1/\delta'))\big)\leq 2c.
%$$
To see this, apply the inequality $e^{\epsilon}-1\leq 2\epsilon$ (easily shown via Taylor's theorem and the assumption that $\epsilon\leq 1$).% and merge the two terms.
\end{remark}
In addition, we will also make use of the following lemma due to \citet{beimel2014bounds}. % and also appeared in the same explicit form in \citet{bassily2014private}:
\begin{lemma}[Privacy for subsampled data. Lemma~4.4 in \citet{beimel2014bounds}.]\label{lemma:subsampling}
Over a domain of data sets $\cX^N$, if an algorithm $\cA$ is $(\epsilon,\delta)$ differentially private (with $\epsilon<1$), then for any data set $X\in \cX^N$, running $\cA$ on a uniform random $\gamma N$-entries of $X$ ensures $(2\gamma \epsilon,\delta)$-DP.
\end{lemma}

% Should we provide the proof for self-containedness? It wasn't completely clear in Bassily's paper.
To make sense of the above lemma, notice that we are subsampling uniform randomly and the probability of any single data point being sampled is only $\gamma$. Thus, if we arbitrarily perturb one of the data points, its impact is evenly spread across all data points thanks to random sampling.

Let $f: \cX^n \rightarrow \R^d$ be an arbitrary $d$-dimensional function. Define the $\ell_2$ sensitivity of $f$ to be
$$\Delta_2 f= \sup_{Y: d(X,Y)\leq 1} \|f(X)-f(Y)\|_2.$$
Suppose we want to output $f(X)$ differential privately, ``Gaussian Mechanism'' output $\hat{f}(X) = f(X) + \cN(0,\sigma^2 I_d)$ for some appropriate $\sigma$.
\begin{theorem}[Gaussian Mechanism, c.f. \citet{dwork2013algorithmic}]\label{lemma:gauss_mech}
Let $\epsilon\in(0,1)$ be arbitrary. ``Gaussian Mechanism'' with $\sigma \geq \Delta_2 f \sqrt{2\log(1.25/\delta)}/\epsilon$ is $(\epsilon,\delta)$-differentially private.
\end{theorem}
This will be the main workhorse that we use here.

\subsection{Stochastic Gradient Langevin Dynamics}

SGLD iteratively update the parameters to by running a perturbed version of the minibatch stochastic gradient descent on the negative log-posterior objective function
$$
 - \sum_{i=1}^{N}\log p(\vct x_i|\vct \theta) -\log \pi(\vct \theta)=: \sum_{i=1}^{N}\ell(\vct x_i;\vct \theta) + r(\vct \theta)
$$
where $\ell(\vct x_i;\vct \theta)$ and $r(\vct \theta)$ are loss-function and regularizer under the empirical risk minimization.

If one were to run stochastic gradient descent or any other optimization tools on this, one would eventually a deterministic maximum a posteriori estimator. SGLD avoids this by adding noise in every iteration.
At iteration $t$ SGLD first samples uniform randomly $\tau$ data points $\{\vct x_{t_1},...,\vct x_{t_2}\}$ and then updates the parameter using
\begin{equation}\label{eq:SGLD}
\vct \theta_{t+1} = \vct \theta_t - \eta_t\left(\nabla r(\vct \theta) +\frac{N}{\tau}\sum_{i=1}^\tau\nabla \ell(\vct x_{ti}|\vct \theta)\right) + \vct z_t,
\end{equation}
where $\vct z_t\sim \cN(0,\eta_t)$ and $\tau$ is the mini-batch size.

For the ordinary stochastic gradient descent to converge in expectation, the stepsize $\eta_t$ can be chosen as anything that $\sum_{i=1}^{\infty} \eta_t= \infty$ and $\sum_{i=1}^{\infty} \eta_t^2 <\infty$ \citep{robbins1951stochastic}. Typically, one can chooses stepsize $\eta_t = a(b+t)^{-\gamma}$ with $\gamma\in(0.5,1]$. In fact, it is shown that for general convex functions and $\mu$-strongly convex functions $\frac{1}{\sqrt{t}}$ and $\frac{1}{\mu t}$ can be used to obtain the minimax optimal $O(1/\sqrt{t})$ and $O(1/t)$ rate of convergence. These results substantiate the first phase of SGLD: a convergent algorithm to the optimal solution. Once it gets closer, however, it transforms into a posterior sampler. According to \citet{welling2011bayesian} and later formally proven in \citet{sato2014approximation}, if we choose $\eta_t \rightarrow 0 $, the random iterates $\vct \theta_t$ of SGLD converges in distribution to the $p(\vct \theta|X)$. The idea is that as the stepsize gets smaller, the stochastic error from the true gradient due to the random sampling of the minibatch converges to $0$ faster than the injected Gaussian noise.

In addition, if we use some fixed stepsize lower bound, such that $\eta_t = \max\{1/(t+1),\eta_0\}$ (to alleviate the slow mixing problem of SGLD), the results correspond to a discretization approximation of a stochastic differential equation (Fokker-Planck equation), which obeys the following theorem due to \citet{sato2014approximation} (simplified and translated to our notation).
\begin{theorem}[Weak convergence \citep{sato2014approximation}]
Assume $f(\vct \theta|X)$ is differentiable, $\nabla f(\vct \theta |X)$ is gradient Lipschitz and bounded \footnote{We use boundedness to make the presentation simpler. Boundedness trivially implies the linear growth condition in \citet[Assumption 2]{sato2014approximation}.}. Then
$$\left|\E_{\vct \theta\sim p(\vct \theta|X)}[h(\vct \theta)] - \E_{\vct \theta \sim SGLD}[h(\vct \theta(t))]\right| = O(\eta_t),$$
for any continuous and polynomial growth function $h$.
\end{theorem}
This theorem implies that one can approximate the posterior mean (and other estimators) using SGLD. Finite sample properties of SGLD is also studied in \citep{vollmer2015non}.

Now we will show that with a minor modification to just the ``burn-in'' phase of SGLD, we will be able to make it differentially private (see Algorithm~\ref{alg:DP-SGLD}).

%Our first result states that if the stepsize is chosen to be sufficiently small, the non-rejecting variant of Langevin Dynamics by itself is $(\epsilon,\delta)$-differentially private.
%\begin{theorem}[Differentially private Langevin Dynamics]
%Suppose we will run $T$ iterations with if stepsize
%$$\eta_t\leq \frac{\epsilon^2}{TL^2 \log(1.25/\delta)\log(T/\delta)}$$
%automatically guarantees $(\epsilon,\delta)$-differential privacy.
%\end{theorem}
%
%%\begin{theorem}[Differentially private Langevin Dynamics]
%%Suppose we will run $T$ iterations with
%%$$\sigma^2 = \frac{\eta_t^2TL^2 \log(1.25/\delta)\log(T/\delta)}{\epsilon^2}  \vee  \eta_t,$$
%%guarantees $(2\epsilon,\delta)$-differential privacy.
%%Moreover, when $\sigma^2=\eta_t$, the procedure reduces to Langevin Dynamics.
%%\end{theorem}
%\begin{proof}
%The output gradient had global $\ell_2$-sensitivity of $L$ by the $L$-Lipschitz gradient assumption.
%By Gaussian Mechanism, each iteration obeys $(\epsilon/\sqrt{T},\delta/T)$-DP. Then by the advanced composition theorem over $T$ oracle calls, we get the total privacy loss $2\epsilon$, with failure probability $\delta$.
%\end{proof}

%To see that it reduces to Langevin dynamics at a reasonable time point, take stepsize $\eta_t=\frac{\epsilon^2}{2L^2\log(1.25/\delta)t}$
%Then when $t>0.5 T$, the algorithm becomes Langevin dynamics and we can spend the second half of the iterations to collect samples.

\begin{algorithm}[b]                     % enter the algorithm environment
\caption{Differentially Private Stochastic Gradient Langevin Dynamics (DP-SGLD)}          % give the algorithm a caption
\label{alg:DP-SGLD}                           % and a label for \ref{} commands later in the document
\begin{algorithmic}                    % enter the algorithmic environment
    \REQUIRE Data $X$ of size $N$, Size of minibatch $\tau$, number of data passes $T$, privacy parameter $\epsilon,\delta$, Lipschitz constant $L$ and  initial $\vct \theta_1$. Set $t=1$.
%    \ENSURE $y = x^n$
	\FOR{$t=1:\lfloor NT/\tau\rfloor$}
	   \STATE{1.} Random sample a minibatch $S\subset [N]$ of size $\tau$.
 	   \STATE{2.} Sample each coordinate of $\vct z_t$ iid from
 	   $ \cN\left( 0, \frac{128NT L^2}{ \tau\epsilon^2}\log\Big(\frac{2.5NT}{\tau\delta}\Big)\log(2/\delta)  \eta_t^2 \vee  \eta_t\right).$
 	   \STATE{3.} Update $\vct \theta_{t+1}\leftarrow \vct \theta_t - \eta_t\left(\nabla r(\vct \theta) +\frac{N}{\tau}\sum_{i\in S}\nabla \ell(\vct x_{i}|\vct \theta)\right) + \vct z_t,$
        \STATE{4.} Return $\vct \theta_{t+1}$ as a posterior sample (after a pre-defined burn-in period).
 	   \STATE{5.} Increment $t\leftarrow t+1.$
    \ENDFOR
\end{algorithmic}
\end{algorithm}
\begin{algorithm}[b]                   % enter the algorithm environment
\caption{Hybrid Posterior Sampling Algorithm}          % give the algorithm a caption
\label{alg:hybrid}                           % and a label for \ref{} commands later in the document
\begin{algorithmic}                    % enter the algorithmic environment
    \REQUIRE Data $X$ of size $N$, $\log$-likelihood function $\ell(\cdot|\theta)$ with Lipschitz constant $L$ in the first argument, assume $\sup_{\vct x\in\cX}\|\vct x\|$, a prior $\pi$. Privacy requirement $\epsilon$.
    \STATE{1. } Run OPS estimator: Algorithm~\ref{alg:OPS} with $\epsilon/2$. Collect sample point $\theta_0$
    \STATE{2. } Run DP-SGLD (Algorithm~\ref{alg:DP-SGLD}) or other Stochatic Gradient Monte Carlo algorithms and collect samples.
 	\OUTPUT: Return all samples.
\end{algorithmic}
\end{algorithm}

\begin{theorem}[Differentially private Minibatch SGLD]\label{thm:private_SGLD}
Assume initial $\vct \theta_1$ is chosen independent of the data, also assume $\ell(\vct x|\vct \theta)$ is $L$-smooth in $\|\cdot \|_2$ for any $\vct x\in \cX$ and $\vct \theta\in \Theta$. In addition, let $\epsilon,\delta, \tau, T$ be chosen such that $T\geq \frac{\epsilon^2 N}{32\tau \log(2/\delta)}$.
Then Algorithm~\ref{alg:DP-SGLD} preserves $(\epsilon,\delta)$-differential privacy.
%running $T$ data passes using $\tau$ random minibatches ($\lfloor NT/\tau\rfloor$ iterations), and adding noise with variance
%$$\sigma^2 = \frac{32NT L^2\log\Big(\frac{2.5NT}{\tau\delta}\Big)\log(2/\delta)}{ \tau\epsilon^2}  \eta_t^2 \vee  \eta_t$$
\end{theorem}
\begin{proof}
In every iteration, the only data access is $\sum_{i\in S}\nabla \ell(\vct x_{i}|\vct \theta)$  and by the $L$-Lipschitz condition, the sensitivity of $\sum_{i\in S}\nabla \ell(\vct x_{i}|\vct \theta)$ is at most $2L$. Get the essential noise that is added to $\sum_{i\in S}\nabla \ell(\vct x_{i}|\vct \theta)$ by removing the $\frac{N^2\eta_t^2}{\tau^2}$ factor from the variance $\sigma^2$ in the algorithm, and Gaussian mechanism, ensures the privacy loss to be smaller than $\frac{\epsilon\sqrt{N}}{\sqrt{32\tau T\log (2/\delta)}}$ with probability $>1-\frac{\tau\delta}{2NT}$.

Using the same technique in \citet{bassily2014private}, we can further exploit the fact that the subset $S$ that we use to compute the stochastic gradient is chosen uniformly randomly. By Lemma~\ref{lemma:subsampling}, the privacy loss for this iteration is in fact
$$\frac{\epsilon\sqrt{N}}{\sqrt{32\tau T\log (2/\delta)}} \cdot \frac{2\tau}{N} = \frac{\epsilon/2}{\sqrt{2(NT/\tau) \log(2/\delta)}}.$$
Verify that we can indeed do that as $\frac{\epsilon\sqrt{N}}{\sqrt{32\tau T\log (2/\delta)}} <1$ from the assumption on $T$. Note that to get $T$ data passes with minibatches of size $\tau$, we need to go through at most $\lfloor \frac{NT}{\tau}\rfloor\leq \frac{NT}{\tau}$ iterations. Apply the advanced composition theorem (Remark~\ref{rmk:adv_composition}), we get an upper bound of the total privacy loss $\epsilon$ and  failure probability $\delta = \frac{\delta}{2} + \frac{\tau\delta}{2NT} \cdot \frac{NT}{\tau}$ accordingly.

The proof is complete by noting that choosing a larger noise level when $\eta_t$ is bigger can only reduces the privacy loss under the same failure probability.
\end{proof}

\paragraph{$\alpha$-Phase transition.} For any $\alpha\in (0,1)$, if we choose $\eta_t = \frac{\alpha \epsilon^2}{128L^2\log(2.5NT/(\tau\delta))\log(2/\delta) t}$, then whenever $t>\alpha NT/\tau$, then we are essentially running SGLD for the last $(1-\alpha)NT/\tau$ iterations, and we can collect approximate posterior samples from there. %Note that this is a theoretically justified stepsize to use when the objective function is at least $\frac{\alpha \epsilon^2}{128L^2\log(2.5NT/(\tau\delta))\log(2/\delta)}$-strongly convex.

\paragraph{Small constant $\eta_0$.}
Instead of making $\eta_t$ to converge to $0$ as $t$ increases, we may alternatively use constant $\eta_0$ after $t$ is larger than a threshold. This is a suggested heuristic in \citet{welling2011bayesian} and is inline with the analysis in \citet{sato2014approximation} and \citet{vollmer2015non}.

\paragraph{Choice of $T$ and $\tau$}
By \citet{bassily2014private}, it takes at least $N$ data passes to converge in expectation to a point near the minimizer, so taking $T=2N$ is a good choice.
 The variance of both random components in our stochastic gradient is smaller when we use larger $\tau$. Smaller variances would improve the convergence of the stochastic gradient methods and make the SGLD a better approximation to the full Langevin Dynamics. The trade-off is that  when $\tau$ is too large, we will use up the allowable $T$ datapasses with just $O(T)$ iterations and the number of posterior samples we collect from the algorithm will be small.

%\paragraph{Flexibility in choosing $\sigma^2$}
%We would like to point out that our choice of $\sigma^2=\eta_t$ is for convenience, as it corresponds to the canonical form of Langevin Dynamics in \citet{welling2011bayesian} and \citet{neal2011mcmc}. Based on a more general formulation in \citet{neal2011mcmc}, however, one is free to choose any scale parameter $\sigma^2:=K \eta_t$ (random walk bandwidth) for any constant $K$ does not affect the asymptotic convergence of the Markov chain sample distribution to the posterior distribution. And as long as $\eta_t\rightarrow 0$, the rejection probability goes to $0$. Essentially, we are just to choose a different ``kinetic energy'' term in the Hamiltonian Monte Carlo and do only one ``Leapfrog'' at a time like the ``Langevin variant''.

\paragraph{Overcoming the large-noise in the ``Burn-in'' phase}
When the stepsize $\eta_t$ is not small enough initially, we need to inject significantly more noise than what SGLD would have to ensure privacy. We can overcome this problem by initializing the SGLD sampler with a valid output of the \OPS estimator, modified according to the exponential mechanism so that the privacy loss is calibrated to $\epsilon/2$. As the initial point is already in the high probability region of the posterior distribution, we no longer need to ``Burn-in'' the Monte Carlo sampler so we can simply choose a sufficiently small constant stepsize so that it remains a valid SGLD. This algorithm is summarized in Algorithm~\ref{alg:hybrid}.

\paragraph{Comparing to \OPS}
The privacy claim of DP-SGLD is very different from \OPS. It does not require sampling to be nearly correct to ensure differential privacy. In fact, DP-SGLD privately releases the entire sequence of parameter updates, thus ensures differential privacy even if the internal state of the algorithm gets hacked. However, the quality of the samples is usually worse than \OPS due to the random-walk like behavior. The interesting fact, however, is that if we run SGLD indefinitely without worrying about the stronger notion of internal privacy, it leads to a valid posterior sample. We can potentially modify the posterior distribution to sample from into the ``scaled'' version so as to balancing the two ways of getting privacy.

%For other instances where a sequence of online updates of the covariance matrix of the gradient, namely, sample Fisher Information, are needed, we will need to modify the algorithm to privatize this part of the problem too.

\subsection{Hamiltonian Dynamics, Fisher Scoring and Nose-Hoover Thermostat}
One of the practical drawback of SGLD is its random walk-like behavior which slows down the mixing significantly. In this section, we describe three extensions of SGLD that attempts to resolve the issue by either using auxiliary variables to counter the noise in the stochastic gradient\citep{Chen2014SGHMC,ding2014bayesian}, or to exploit second order information so as to use Newton-like updates with large stepsize \citep{ahn2012bayesian}.

%appeal to the asymptotic normality of certain posterior distribution so as to use larger stepsizes .
We note that in all these methods, stochastic gradients are the only form of data access, therefore similar results like what we described for SGLD follow nicely. We briefly describe each method and how to choose their parameters for differential privacy.

\paragraph{Stochastic Gradient Hamiltonian Monte Carlo.}
 According to \citet{neal2011mcmc}, Langevin Dynamics is a special limiting case of Hamiltonian Dynamics, where one can simply ignore the ``momentum'' auxiliary variable.  In its more general form, Hamiltonian Monte Carlo (HMC) is able to generate proposals from distant states and hence enabling more rapid exploration of the state space.
\citet{Chen2014SGHMC} extends the full ``leap-frog'' method for HMC in \citet{neal2011mcmc} to work with stochastic gradient and add a ``friction'' term in the dynamics to ``de-bias'' the noise in the stochastic gradient.
\begin{equation}\label{eq:SGHMC}
  \left\{
    \begin{array}{ll}
      \vct \theta_t = \vct \theta_{t-1} + h_t  \vct r_{t-1}\\
      \vct p_t = \vct p_{t-1} - h_t \widehat{\nabla} -\eta_t A \vct p_{t-1} + \cN(0, 2(A-\widehat{B}) h_t).
    \end{array}
  \right.
\end{equation}
where $\widehat{B}$ is a guessed covariance of the stochastic gradient (the authors recommend  restricting $\hat{B}$ to a single number or a diagonal matrix) and
$A$ can be arbitrarily chosen as long as $A\succ \widehat{B}$. If the stochastic gradient $\widehat{\nabla}\sim \cN(\nabla, B)$ for some $B$ and $\widehat{B}=B$, then this dynamics is simulating a dynamic system that yields the correct distribution. Note that even if the normal assumption holds and we somehow set $\widehat{B}=B$, we still requires $h_t$ to go to $0$ to sample from the actual posterior distribution, and as $h_t$ converges to $0$ the additional noise we artificially inject dominates and we get privacy for free. All we need to do is to set $A$, $\widehat{B}$ and $h_t$ so that
$2(A-\widehat{B}) /h_t\succ \frac{128NT L^2}{ \tau\epsilon^2}\log\Big(\frac{2.5NT}{\tau\delta}\Big)\log(2/\delta)  I_n$.
Note that as $h_t\rightarrow 0$ this quickly becomes true.

\paragraph{Stochastic Gradient Nos\'{e}-Hoover Thermostat}
As we discussed, the key issue about SGHMC is still in choosing $\widehat{B}$. Unless $\widehat{B}$ is chosen exactly as the covariance of true stochastic gradient, it does not sample from the correct distribution even as $h_t\rightarrow 0$ unless we trivially set $\hat{B}=0$. The Stochastic Gradient Nos\'{e}-Hoover Thermostat (SGNHT) overcomes the issue by introducing an additional auxiliary variable $\xi$, which serves as a thermostat to absorb the unknown noise in the stochastic gradient. The update equations of SGNHT are given below
\begin{equation}\label{eq:SGNHT}
  \left\{
    \begin{array}{ll}
      \vct p_t = \vct p_{t-1} - \xi_{t-1}\vct p_{t-1}h_t - \widehat{\nabla}h_t + \cN(0, 2Ah_t);\\
      \vct \theta_t = \vct \theta_{t-1} + h_t \vct p_{t-1};\\
      \xi_t = \xi_{t-1} + (\frac{1}{n} \vct p_t^T\vct p_t -1) h_t.
    \end{array}
  \right.
\end{equation}
Similar to the case in SGHMC, appropriately selected discretization parameter $h_t$ and the friction term $A$ will imply differential privacy.

\citet{Chen2014SGHMC,ding2014bayesian} both described a reformulation that can be interpret as SGD with momentum. This is by setting parameters  $\eta=h^2, a=hA, \hat{b}=h \widehat B$ for SGHMC:
\begin{equation}\label{eq:SGHMC_reform}
  \left\{
    \begin{array}{ll}
      \vct \theta_t = \vct \theta_{t-1} + \vct v_{t-1}\\
      \vct v_t = \vct v_{t-1} - \eta_t \widehat{\nabla} - a\vct v + \cN(0, 2(a-\widehat{b}) \eta_tI);
    \end{array}
  \right.
\end{equation}
and $\vct v = \vct p h, \eta_t=h_t^2,\alpha=\xi h$ and $a = Ah$ for SGNHT:
\begin{equation}\label{eq:SGNHT_reform}
  \left\{
    \begin{array}{ll}
      \vct v_t = \vct v_{t-1} - \alpha_{t-1}\vct v_{t-1} - \widehat{\nabla}(\vct \theta_{t-s})\eta_t + \cN(0, 2 a\eta_t I);\\
      \vct \theta_t = \vct \theta_{t-1} + \vct u_{t-1};\\
      \alpha_t = \alpha_{t-1} + (\frac{1}{n} \vct v_t^T\vct v_t -\eta_t).
    \end{array}
  \right.
\end{equation}
where $1-a$ is the momentum parameter and $\eta$ is the learning rate in the SGD with momentum.
Again note that to obtain privacy, we need $\frac{2a}{\eta_t} \geq \frac{128NT L^2}{\tau\epsilon^2}\log(\frac{2NT}{\tau\delta})\log(1/\delta)$.

Note that as $\eta_t$ gets smaller, we have the flexibility of choosing $a$ and $\eta_t$ within a reasonable range.

\paragraph{ Stochastic Gradient Fisher Scoring}
Another extension of SGLD is Stochastic Gradient Fisher Scoring (SGFS), where \citet{ahn2012bayesian} proposes to adaptively interpolate between a preconditioned SGLD (see preconditioning \citep{girolami2011riemann}) and a Markov Chain that samples from a normal approximation of the posterior distribution. For parametric problem where Bernstein-von Mises theorem holds, this may be a good idea. The heuristic used in the SGFS is that the covariance matrix of $\vct \theta|X$, which is also the inverse Fisher information $I_N^{-1}$ is estimated on the fly. The key features of SGFS is that one can use the stepsize to trade off speed and accuracy, when the stepsize is large, it mixes rapidly to the normal approximation, as the stepsize gets smaller the stationary distribution converges to the true posterior. Further details of SGFS and ideas to privatize it is described in the appendix.

\subsection{Discussions and caveats.}
So far, we have proposed a differentially private Bayesian learning algorithm that is memory efficient, statistically near optimal for a large class of problems, and we can release many intermediate iterates to construct error bars. Given that differential privacy is usually very restrictive, some of these results may appear too good to be true. This is a reasonable suspicion due to the following caveats.

\paragraph{Small $\eta$ helps both privacy and accuracy.} It is true that as $\eta$ goes to $0$, the stationary distribution that these method samples from gets closer to the target distribution. On the other hand, since the variance of the noise we need to add for privacy scales in $O(\eta^2)$ and that for posterior sampling scales like $O(\eta)$, privacy and accuracy benefits from the same underlying principle. The caveat is that we also have a budget on how many samples can we collect. Also the smaller the stepsize $\eta$ is, the slower it mixes, as a result, the samples we collect from the monte carlo sampler is going to be more correlated to each other.

\paragraph{Adaptivity of SGNHT.} While SGNHT is able to adaptively adjust the temperature so that the samples that it produces remain ``unbiased'' in some sense as $\eta\rightarrow 0$. The reality is that if the level of noise is too large, either we adjust the stepsize to be too small to search the space at all, or the underlying stochastic differential equation becomes unstable and quickly diverges. As a result, the adaptivity of SGNHT breaks down if the privacy parameter gets to small.

\paragraph{Computationally efficiency.} For a large problem, it is usually the case that we would like to train with only one pass of data or very small number of data passes. However, due to the condition in Lemma~\ref{lemma:subsampling}, our result does not apply to one pass of data unless $\tau$ is chosen to be as large as $N$. While we can still choose $T$ to be sufficiently large and stop early, but we amount of noise that we add in each iteration will remain the same.

\paragraph{The Curse of Numerical constant.}
The analysis of algorithms often involves larger numerical constants and polylogarithmic terms in the bound. In learning algorithms these are often fine because there are more direct ways to evaluate and compare methods' performances. In differential privacy however, constants do matter. This is because we need to use these bounds (including constants) to decide how much noise or perturbation we need to inject to ensure a certain degree of privacy. These guarantees are often very conservative, but it is intractable to empirically evaluate the actual $\epsilon$ of differential privacy due to its ``worst'' case definition. Our stochastic gradient based differentially private sampler suffers from exactly that. For moderate data size, the product of the constant and logarithmic terms can be as large as a few thousands. That is the reason why it does not perform as well as other methods despite the theoretically being optimal in scaling (the optimality result is due to SGD \citep{bassily2014private}).

%\section{Hybrid DP mechanisms for Differentially Private Bayesian Learning }
%
%Posterior sampling + Approximate MCMC.
%
%
%Benefits of Stochastic Gradient approach. Memory efficient, many variants, readily scalable to Big Data.
%Helps in differentially private joint training.

\section{Experiments}

Figure~\ref{fig:illus} is a plain illustration of how these stochastic gradient samplers work using a randomly generated linear regression model (note the its posterior distribution will be normal, as the contour illustrates). On the left, it shows how these methods converge like stochastic gradient descent to the basin of convergence. Then it becomes a posterior sampler. The figure on the right shows that the stochastic gradient thermostat is able to produce more accurate/unbiased result and the impact of differential privacy at the level of $\epsilon=10$ becomes negligible.
\begin{figure}
  \centering
  \includegraphics[width=0.45\textwidth]{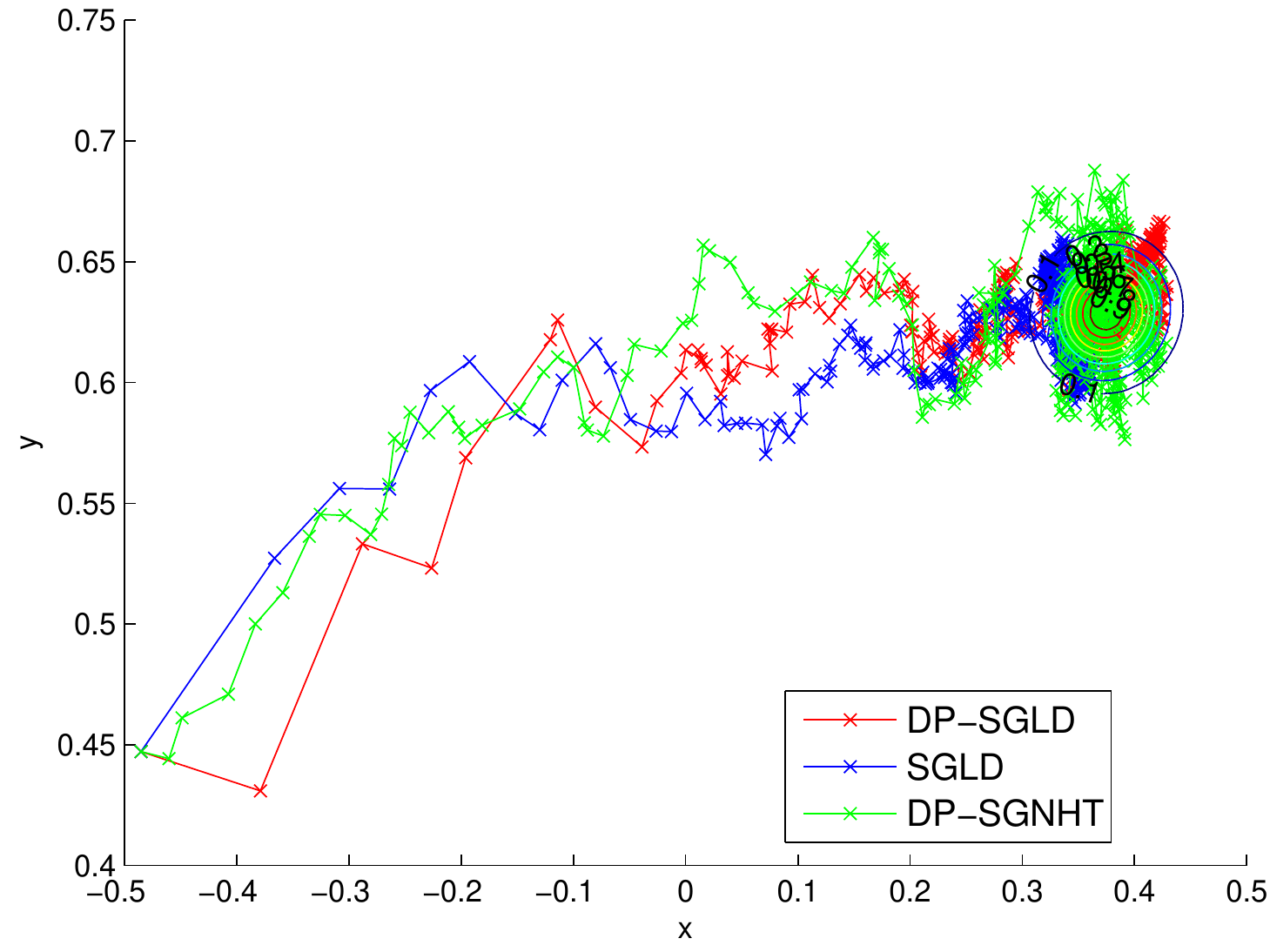}
  \includegraphics[width=0.45\textwidth]{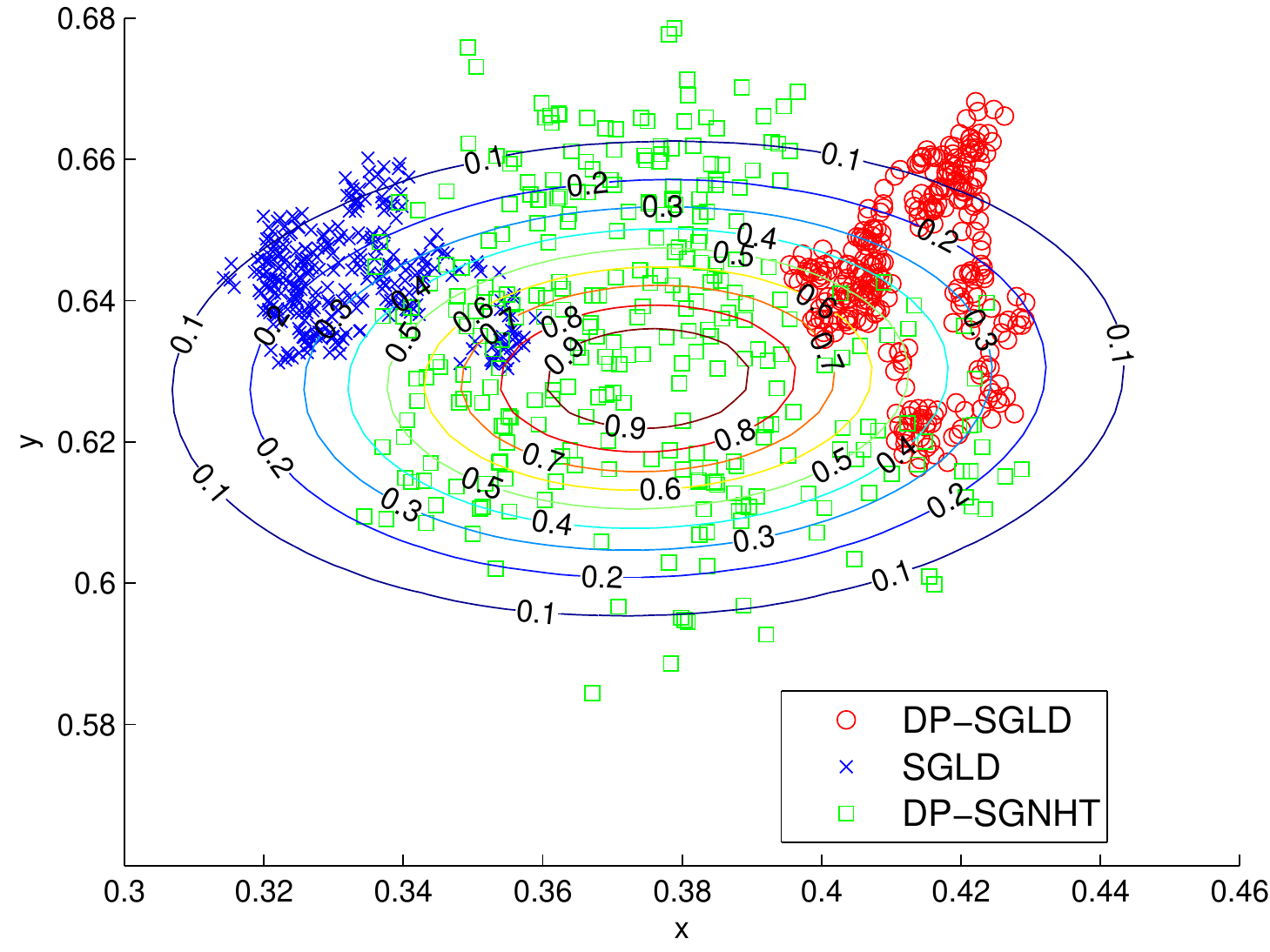}\\
  \caption{Illustration of stochastic gradient langevin dynamics and its private counterpart at $\epsilon=10$.}\label{fig:illus}
\end{figure}

\begin{figure}
  \centering
  \subfigure[Synthetic: classification of two normals.]{
  \includegraphics[width=0.48\textwidth]{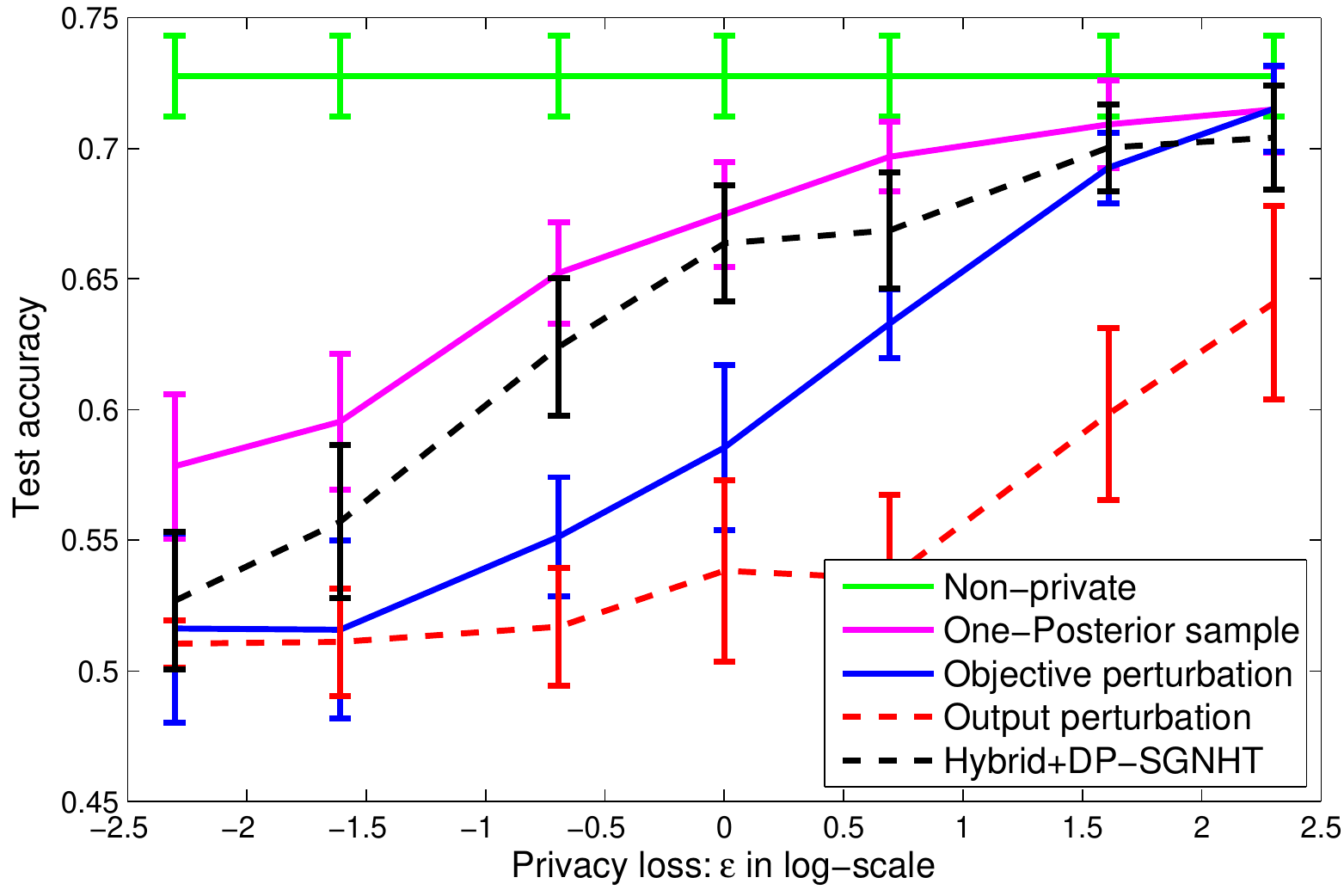}}
  \subfigure[Abalone: 9 features, 4177 data points.]{
     \includegraphics[width=0.48\textwidth]{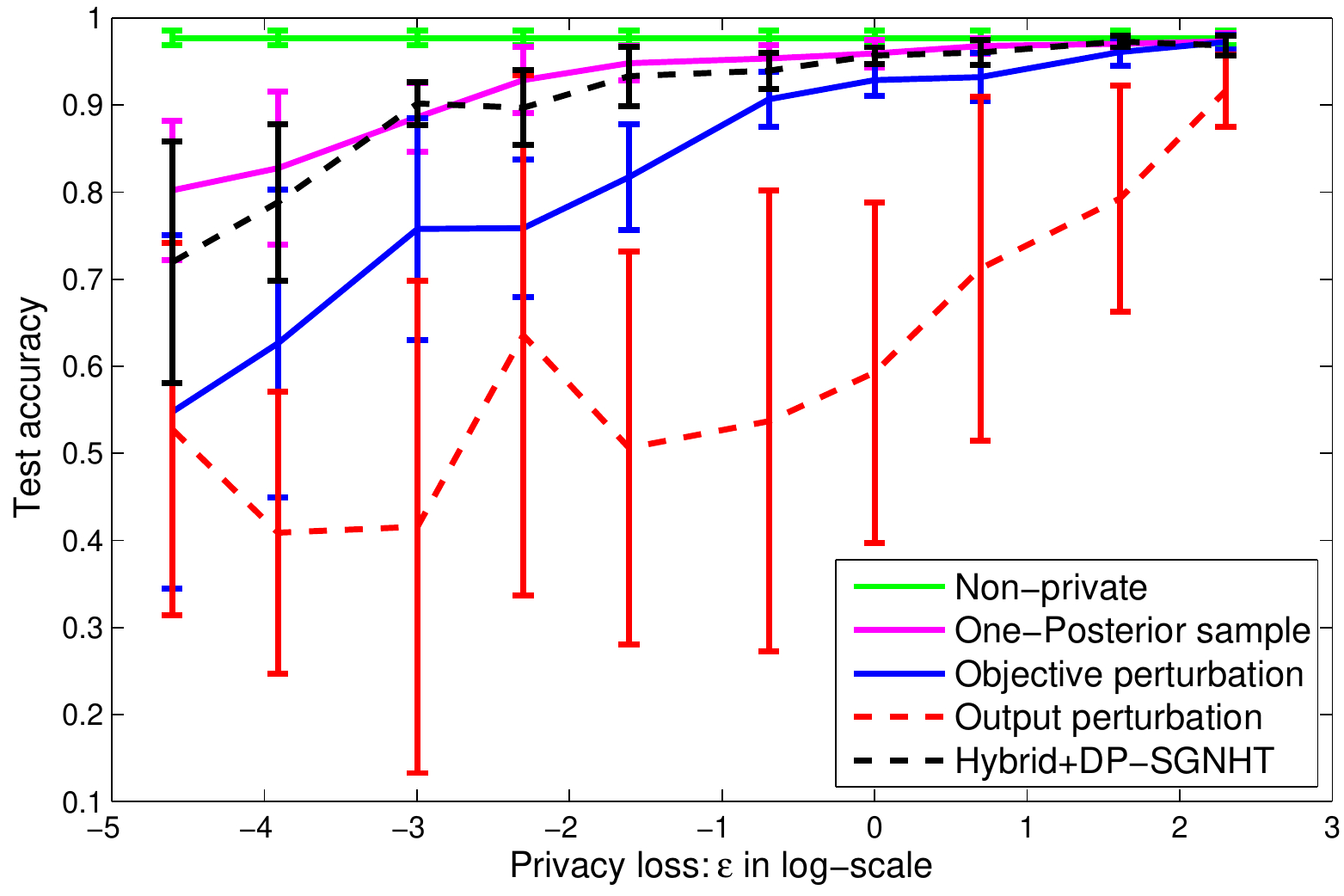}}
     \subfigure[Adult: 109 features, 32561 data points.]{
    \includegraphics[width=0.48\textwidth]{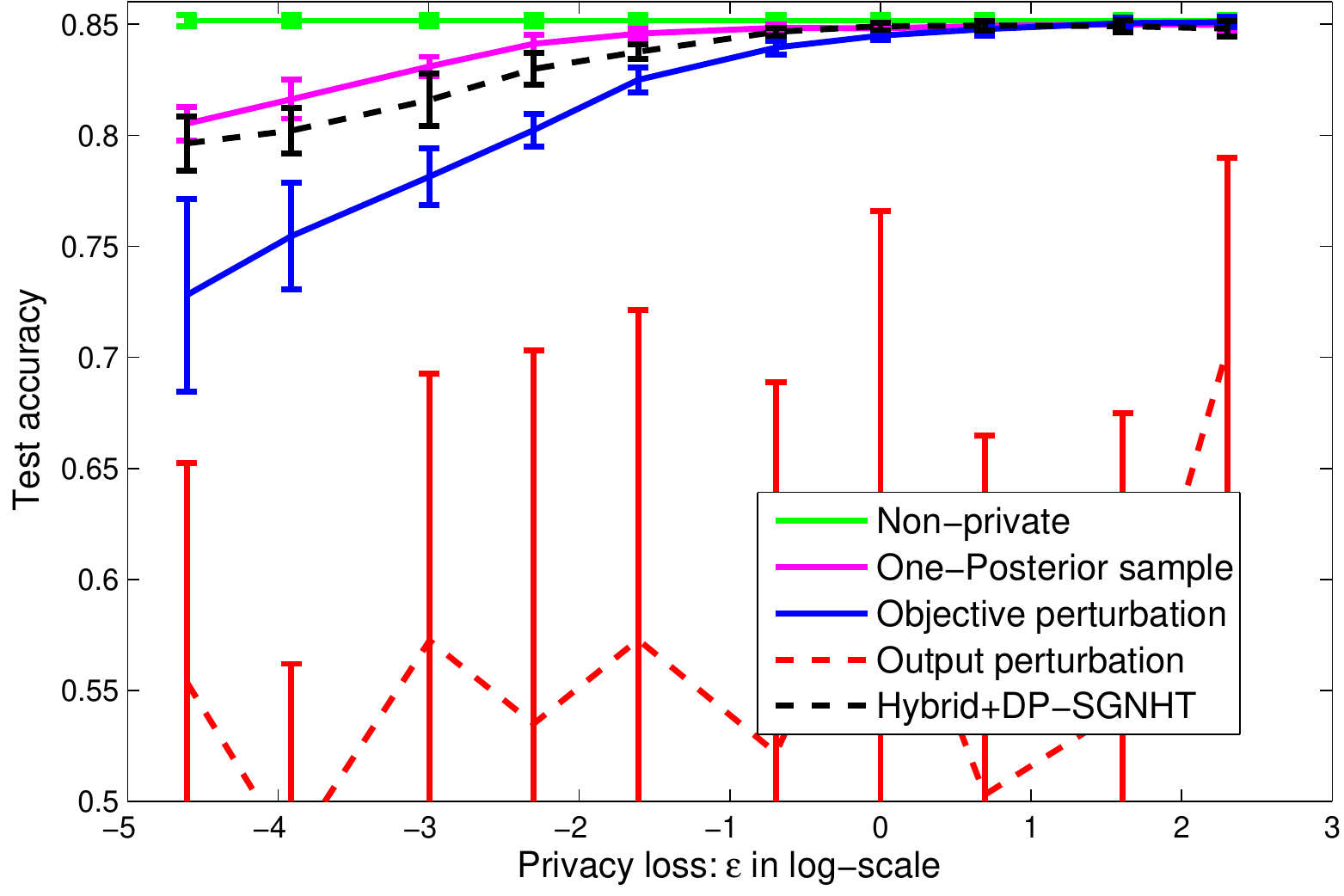}}
  \caption{Comparison of Differential Private methods.}\label{fig:results}
\end{figure}

To evaluate how our proposed methods work in practice, we selected two binary classification datasets: Abalone and Adult, from the first page of UCI Machine Learning Repository and performed privacy constrained logistic regression on them. Specifically, we compared two of our proposed methods, \OPS mechanism and hybrid algorithm against the state-of-the-art empirical risk minimization algorithm \OBP \citep{chaudhuri2011differentially,kifer2012private} under varying level of differential privacy protection. The results are shown in Figure~\ref{fig:results}. As we can see from the figure, in both problems, \OPS significantly improves the classification accuracy over \OBP. The hybrid algorithm also works reasonably well, given that it collected $N$ samples after initializing it from the output of a run of \OPS with privacy parameter $\epsilon/2$. For fairness, we used the $(\epsilon,\delta)$-DP version of the objective perturbation \citep{kifer2012private} and similarly we used Gaussian mechanism (rather than Laplace mechanism) for output perturbation. All optimization based methods are solved using BFGS algorithm to high numerical accuracy. \OPS is implemented using SGNHT and we ran it long enough so that we are confident that it is a valid posterior sample. Minibatch size and number of data passes in the hybrid DP-SGNHT are chosen to be both $\sqrt{N}$.

We note that the plain DP-SGLD and DP-SGNHT without an initialization using \OPS does not work nearly as well. In our experiments, it often performs equally or slightly worse than the output perturbation. This is due to the few caveats (especially ``the curse of numerical constant'') we described earlier.

\section{Related work}\label{sec:relatedwork}
We briefly discuss related work here.
For the first part, we become aware recently that \citet{mir13} and \citet{dimitrakakis2014robust} independently developed the idea of using posterior sampling for differential privacy. \citet[Chapter 5]{mir13} used a probabilistic bound of the log-likelihood to get $(\epsilon,\delta)-DP$ but focused mostly on conjugate priors where the posterior distribution is in closed-form. \citet{dimitrakakis2014robust} used Lipschitz assumption and bounded data points (implies our boundedness assumption) to obtain a generalized notion of differential privacy. Our results are different in that we also studied the statistical and computational properties.
\citet{bassily2014private} used exponential mechanism for empirical risk minimization and the procedure is exactly the same as \OPS.  Our difference is to connect it to Bayesian learning and to provide results on limiting distribution, statistical efficiency and approximate sampling. We are not aware of a similar asymptotic distribution with the exception of \citet{smith2008efficient}, where a different algorithm (the subsample-and-aggregate scheme) is proven to give an estimator that is asymptotically normal and efficient (therefore, stronger than our result) under a different set of assumptions. Specifically, \citet{smith2008efficient}'s method requires boundedness of the parameter space while ours method can work with potentially unbounded space so long as the log-likelihood is bounded.

Related to the general topic,  \citet{kasiviswanathan2014semantics} explicitly modeled the ``semantics'' of differential privacy from a Bayesian point of view, \citet{xiao2012bayesian} developed a set of tools for performing Bayesian inference under differential privacy, e.g., conditional probability and credibility intervals. \citet{williams2010probabilistic} studied a related but completely different problem that uses posterior inference as a meta-post-processing procedure, which aims at ``denoising'' the privately obfuscated data when the private mechanism is known. Integrating \citet{williams2010probabilistic} with our procedure might lead to some further performance boost, but investigating its effect is beyond the scope of the current paper.

For the second part, the idea to privately release stochastic gradient has been well-studied. \citet{song2013stochastic,bassily2014private} explicitly used it for differentially private stochastic gradient descent. And \citet{rajkumar2012differentially} used it for private multi-party training.
Our Theorem~\ref{thm:private_SGLD} is a simple modification of Theorem~2.1 in \citet{bassily2014private}. \citet{bassily2014private} also showed that the differential private SGD using Gaussian mechanism with $\tau=1$ matches the lower bound up to constant and logarithmic, so we are confident that not many algorithms can do significantly better than Algorithm~\ref{alg:DP-SGLD}. Our contribution is to point out the interesting algorithmic structures of SGLD and extensions that preserves differential privacy. The method in \citet{song2013stochastic} requires disjoint minibatches in every data pass, and it requires adding significantly more noise in settings when Lemma~\ref{lemma:subsampling} applies. \citet{song2013stochastic} are however applicable when we are doing only a small number of data passes and for these cases, it gets a much better constant. \citet{rajkumar2012differentially}'s setting is completely different as it injects a fixed amount of noise to the gradient corresponds to each data point exactly once. In this way, it replicates objective perturbation \citep{chaudhuri2011differentially} (assuming the method actually finds the optimal solution).

Objective perturbation is originally proposed in \citet{chaudhuri2011differentially} and the $(\epsilon,\delta)$ version that we refer to first appears in \citet{kifer2012private}. Comparing to our two mechanisms that attempts to sample from the posterior, their privacy guarantee requires the solution to be exact while ours does not. In comparison, \OPS estimator is differentially private allows the distribution it samples from to be approximate, DP-SGLD on the other hand releases all intermediate results and every single iteration is public.

\section{Conclusion and future work}
In this paper, we described two simple but conceptually interesting examples that Bayesian learning can be inherently differentially private. Specifically, we show that getting one sample from the posterior is a special case of exponential mechanism and this sample as an estimator is near-optimal for parametric learning. On the other hand, we illustrate that the algorithmic procedures of stochastic gradient Langevin Dynamics (and variants) that attempts to sample from the posterior also guarantee differential privacy as a byproduct. Preliminary experiments suggests that the One-Posterior-Sample mechanism works very well in practice and it substantially outperforms earlier privacy mechanism in logistic regression. While suffering from a large constant, our second method is also theoretically and practically meaningful in that it provides privacy protection in intermediate steps.

To carry the research forward, we think it is important to identify other cases when the existing randomness can be exploited for privacy. Randomized algorithms such as hashing and sketching, dropout and other randomization used in neural networks might be another thing to look at. More on the application end, we hope to explore the one-posterior sample approach in differentially private movie recommendation. Ultimately, the goal is to make differential privacy more practical to the extent that it can truly solve the real-life privacy problems that motivated its very advent.

%% Acknowledgements should only appear in the accepted version.
%\section*{Acknowledgments}
%
%\textbf{Do not} include acknowledgements in the initial version of
%the paper submitted for blind review.
%
%If a paper is accepted, the final camera-ready version can (and
%probably should) include acknowledgements. In this case, please
%place such acknowledgements in an unnumbered section at the
%end of the paper. Typically, this will include thanks to reviewers
%who gave useful comments, to colleagues who contributed to the ideas,
%and to funding agencies and corporate sponsors that provided financial
%support.

\appendix
%In this Appendix, we provide additional results that do not fit into the main text. Including the derivation that extends the  stochastic gradient fisher scoring to differential private setting.

\section{Stochastic Gradient Fisher Scoring}

\subsection{Fisher Scoring and Stochastic Gradient Fisher Scoring}

Fisher scoring is simply the Newton's method for solving maximum likelihood estimation problem. The score function $S(\theta)$ is the gradient of the $\log$-likelihood. So intuitively, if we solve the equation $S(\theta)=0$, we can obtain the maximum likelihood estimate. Often this equation is highly non-linear, so we consider the an iterative update for the linearized score function (or a quadratic approximation of the likelihood) by Taylor expand it at the current point $\theta_0$
$$S(\theta) \approx S(\theta_0) + I(\theta_0)(\theta-\theta_0)$$
where $I(\theta_0) = -\sum_{i=1}^n \nabla\nabla^T \ell(Z_i;\theta)$ is the observed Fisher information evaluated at $\theta_0$.

By the fact that $S(\theta^*)=0$, and plug in the above equation, we get
$\theta^* = \theta_0 + I^{-1}(\theta_0) S(\theta_0)$
Note that this is a fix point iteration and it gives us an iterative update rule to search for $\theta^*$ via
$$
\theta_{k+1} = \theta_k + I^{-1}(\theta_k)S(\theta_k).
$$
Recall that $S$ is the gradient of the score function and $I^{-1}$ is the covariance of the score function and (under mild regularity conditions) the Hessian of the log-likelihood. As a result, this is often the same as Newton iterations.

An intuitive idea to avoid passing the entire dataset in every iteration is to simply replacing the gradient (the score function) with stochastic gradient and somehow estimate the Fisher information. Stochastic Gradient Fisher Scoring can be thought of as a Quasi-Newton method.

\subsection{Privacy extension}

By invoking a more advanced version of the Gaussian Mechanism, we will show that similar privacy guarantee can be obtained for a modified version of SGFS (described in Algorithm~\ref{alg:DP-SGFS}) while preserving its asymptotic properties. Specifically, under the assumption that $I_N$ is given, when $\eta_t$ is big, it also samples from a normal approximation (with larger variance), when $\eta_t$ is small, the private algorithm becomes exactly the same as SGFS. Moreover, for a sequence of samples from the posterior, the online estimate in the Fisher Information converges an $O(1/N)$ approximation of true Fisher Information as in \citet[Theorem~1]{ahn2012bayesian}.

%Besides adding more noise in the burn-in period as in SGLD, we also injected more noise to the online estimation of the Fisher Information matrix get privacy.

The privacy result relies on a more specific smoothness assumption.
Assume that for any parameter $\theta\in \R^d$, and $X\in \cX^N$ the ellipsoid $E=F B^{d}$ defined by transforming the unit ball $B^{d}$ using a linear map $F$ contains the symmetric polytope spanned by $\{\pm\nabla \ell(x_1,\theta),...,\pm\nabla\ell(x_N,\theta)\}$. From a differential private point of view, this implies that $\nabla_\theta \ell(x,\theta)$'s sensitivity is different towards different direction. Then the non-spherical gaussian mechanism states
\begin{lemma}[Non-Spherical Gaussian Mechanism]
Output $\sum_{i=1}^{N}\nabla\ell(x_i,\theta) + F w$ where $w\sim \cN(0,\frac{(1+\sqrt{1\log(1/\delta)})^2}{\epsilon^2} I_d)$ obeys $(\epsilon,\delta)$-DP.
\end{lemma}

\begin{algorithm}                      % enter the algorithm environment
\caption{Differentially Private Stochastic Gradient Fisher Scoring (DP-SGFS)}          % give the algorithm a caption
\label{alg:DP-SGFS}                           % and a label for \ref{} commands later in the document
\begin{algorithmic}                    % enter the algorithmic environment
    \REQUIRE Data $X$ of size $N$, Size of minibatch $\tau$, number of data passes $T$,
stepsize $\eta_t$ for $t=1,...,\lfloor NT/\tau\rfloor$, a public Lipschitz matrix $F$, and initial $\theta_1$. Set $t=1$, $\sigma^2 =\frac{32T\log(2.5NT/\tau\delta)\log(2/\delta)}{N\tau\epsilon^2}$
	\FOR{$t=1:\lfloor NT/\tau\rfloor$}
	   \STATE{1.} Random sample a minibatch $S\subset [N]$ of size $\tau$, compute $\bar{g} = \frac{1}{\tau} \sum_{i\in S}\nabla \ell(x_{i}|\theta).$
	   \STATE{2.} Sample $Z_t \sim \cN(0,  \sigma^2 \vee \frac{1}{N^2\eta_t} I_d)$, $W_{ij} \sim \cN(0,49\|F\|^4\sigma^2)$.
	   \STATE{3.} Compute private stochastic gradient and sample covariance matrix
\begin{align*}
  \tilde{g} = \bar{g} + F Z_t, && \text{and}&&
   V = \cP_{S^d_{+}}\left\{\frac{1}{\tau-1}\sum_{i\in S} \left\{\nabla\ell_i(\theta_t) - \bar{g}\right\}\left\{\nabla\ell_i(\theta_t) - \bar{g}\right\}^T + W\right\}.
\end{align*}
 	   \STATE{4.} Update the guessed Fisher Information estimate
    $\hat{I}_t = (1-\kappa_t)\hat{I}_{t-1} + \kappa_t  V $. 	
 	   \STATE{5.} Update and return $\theta_{t+1}\leftarrow \theta_t + 2\left( \frac{ (\tau+N)N}{\tau} \hat{I}_t + \frac{4FF^T}{\eta_t} \right)^{-1}\left(\nabla r(\theta_t) +N \tilde{g}\right).$
 	   \STATE{6.} Increment $t\leftarrow t+1.$
    \ENDFOR
\end{algorithmic}
\end{algorithm}

\begin{theorem}
Let $F$ be that $\ell(x;\theta')\leq \ell_{\theta} + \nabla\ell(x;\theta)^T(\theta'-\theta) + \frac{1}{2}\|F(\theta'-\theta)\|^2$ for any $x\in \cX,\theta\in\Theta$. Moreover, let $\epsilon,\delta, \tau, T$ be chosen such that $T\geq \frac{\epsilon^2 N}{32\tau \log(2/\delta)}$. Then Algorithm~\ref{alg:DP-SGFS} guarantees $(2\epsilon,2\delta)$-differential privacy.
\end{theorem}
\begin{proof}
First of all, $\|F\|_2$ is an upper bound for any $\nabla \ell(x|\theta)$, so by applying Lemma~\ref{lem:CovSensitivity} on the every set of subsamples in each iteration, by Gaussian mechanism (Lemma~\ref{lemma:gauss_mech}) and the invariance to post-processing, we know that $V$ is a private release.
Then the proof follows by the same line of argument (subsampling and advanced composition) as in Theorem~\ref{thm:private_SGLD} for $\tilde{g}$ and $V$ respectively, then the result follows by applying the simple composition theorem.
\end{proof}

\begin{lemma}[Sensitivity of the sample covariance operator]\label{lem:CovSensitivity}
Let $\|x\|\leq L$ for any $x\in \cX$, $n>4$, then
$$\sup_{k, x_1,...,x_n, x_k^\prime}\|\widehat{\Cov}(x_1,...,x_k,...,x_n) - \widehat{\Cov}(x_1,...,x_k^\prime,...,x_n)\|_F \leq  \frac{7 L^2}{n-1}.$$
\end{lemma}
\begin{proof}
%$$\widehat{\Cov}(x_1,...,x_n) = \frac{1}{n}\sum_{i=1}^n x_ix_i^T - (\frac{1}{n}\sum_{i=1}^n x_i)(\frac{1}{n}\sum_{i=1}^n x_i)^T$$
We prove by taking the difference of two adjacent covariance matrices and bound the residual.
\begin{align*}
  \widehat{\Cov}(X') =& \widehat{\Cov}(X) + \frac{1}{n-1}(xx^T - x'[x']^T)  +\frac{1}{n(n-1)} (xx^T + x'[x']^T -x[x']^T -x' x^T)\\ &-\frac{1}{n-1} \mu (x-x')^T -\frac{1}{n-1}(x-x') \mu^T.
\end{align*}
Now assume $n>4$ and take the upper bound of every term, we get $\Delta_2 \left(\Cov(X)\right) \leq \frac{7 L^2}{n-1}$.
\end{proof}

\newpage
\bibliographystyle{apa-good}
\bibliography{FreeDP}
\end{document}